\documentclass{article}

\usepackage{PRIMEarxiv}

\usepackage[utf8]{inputenc} 
\usepackage[T1]{fontenc}    
\usepackage{hyperref}       
\usepackage{url}            
\usepackage{booktabs}       
\usepackage{amsfonts}       
\usepackage{nicefrac}       
\usepackage{microtype}      
\usepackage{lipsum}
\usepackage{fancyhdr}       
\usepackage{graphicx}
\usepackage{ulem}
\usepackage{subcaption}
\usepackage{amsthm}
\usepackage{appendix}
\usepackage{algorithmic}
\usepackage{algorithm}

\graphicspath{{media/}}     

\title{Revolutionizing Long-Term Memory in AI: New Horizons with High-Capacity and High-Speed Storage
}

\author{
  Hiroaki Yamanaka, Daisuke Miyashita, Takashi Toi, Asuka Maki, Taiga Ikeda, Jun Deguchi \\
  AI \& System Research Center,  
  Kioxia Corporation, Yokohama, Japan\\
}

\begin{document}
\maketitle

\begin{abstract}
Driven by our mission of ``uplifting the world with memory,'' this paper explores the design concept of ``memory'' that is essential for achieving artificial superintelligence (ASI).
Rather than proposing novel methods, we focus on several alternative approaches whose potential benefits are widely imaginable, yet have remained largely unexplored. 
The currently dominant paradigm, which can be termed ``extract then store,'' involves extracting information judged to be useful from experiences and saving only the extracted content.
However, this approach inherently risks the loss of information, as some valuable knowledge particularly for different tasks may be discarded in the extraction process.
In contrast, we emphasize the ``store then on-demand extract'' approach, which seeks to retain raw experiences and flexibly apply them to various tasks as needed, thus avoiding such information loss.
In addition, we highlight two further approaches: discovering deeper insights from large collections of probabilistic experiences, and improving experience collection efficiency by sharing stored experiences.
While these approaches seem intuitively effective, our simple experiments demonstrate that this is indeed the case.
Finally, we discuss major challenges that have limited investigation into these promising directions and propose research topics to address them.
\end{abstract}


\section{Introduction}
Memory is the process of storing and retrieving information, and is foundation for humans for growth and effective interaction with the world. 
Also for AI agents, memory plays a crucial role toward artificial general intelligence (AGI) that autonomously explores and learns from the real world. 
While parametric memory of original large language models (LLMs) is static, incorporating LLMs with external memory is a common way to enhance self-evolving capability of AI agents~\cite{zhang_survey_2024,wu_human_2025,gao_survey_2025}. 
External memory retains the agent's experiences, including any received information. Retrieval from the memory allows agents to recall useful information beyond parametric memory and a current task context, and to make more accurate and intelligent decisions.

One of primary applications of self-evolving AI agent memory is personalization that adjusts behaviors and responses of agents to elevate user experiences. 
Past conversation history including a profile and preference of the user is retrieved from memory and the retrieved information is provided to prompts to be considered in reasoning by LLM. 
For instance, ChatGPT Memory~\cite{openai_memory_2024} stores conversation history with a user to provide personalized answers. 
SecondMe~\cite{wei_ai-native_2025} stores user's personal information to create an LLM-powered faithful avatar.

Another prominent application of self-evolving AI agent memory is improving the ability to adapt to specific environments and tasks, as well as the ability to generalize across a wide range of environments and tasks over time.
Explored domains range from general-purpose domains, particularly versatile digital assistants, to specialized domains such as coding, graphical user interfaces, finance, medicine, and education~\cite{gao_survey_2025}. 
Among numerous examples, we provide only three specific cases.
ExpeL~\cite{zhao_expel_2024} is a text-based general-purpose agent that extracts common rules from trajectories representing interactions between the agent and the task environment and uses them to improve performance in future tasks.
MIRIX~\cite{wang_mirix_2025} proposes a comprehensive memory architecture to manage experience memory in six useful memory types such as episodic memory, semantic memory, while supporting visual and multimodal experiences.
RoboMemory~\cite{lei_robomemory_2025} extends human-inspired memory mechanisms for physical AI through a hierarchical and parallel memory architecture that supports long-horizon planning and interactive environmental learning.
A common feature across these domains is learning from experiences. 

In addition, we emphasize another common design concept shared by most existing methods: the ``extract then store'' approach, in which useful information is extracted from experiences and retained as memory.
Typically, usefulness is evaluated with respect to the task being addressed when the experiences are acquired.
Consequently, information not deemed useful is inevitably discarded and not stored as memory.
Biological brains, such as those of humans, exhibit the ability known as latent learning.
For example, in a study of rats exploring mazes~\cite{tolman_cognitive_1948}, rats can acquire the locations of food and water while wandering the maze without specific motivation to find them. Later, when hungry or thirsty, they are able to recall these locations.
This capability is unattainable under the ``extract then store'' paradigm.

Extracting and storing everything that might be useful from each experience leads to unnecessary redundancy.
A natural alternative is to retain raw experiences as memory and extract relevant information on demand according to the current task.
We refer to this approach as ``Store Then ON-demand Extract: STONE'' paradigm.
While it is intuitive that storing raw experiences via STONE minimizes information loss, most prior methods do not adopt this paradigm.
This is largely due to significant challenges, such as increased memory requirements, difficulty in retrieving relevant information, a long response time for user's requests, etc.
Furthermore, commonly used AI benchmarks rarely require leveraging past experiences in contexts different from those in which they were acquired, reducing the incentive to use such experiences later.

\begin{figure}[t]
  \begin{center}
    \includegraphics[scale=.5]{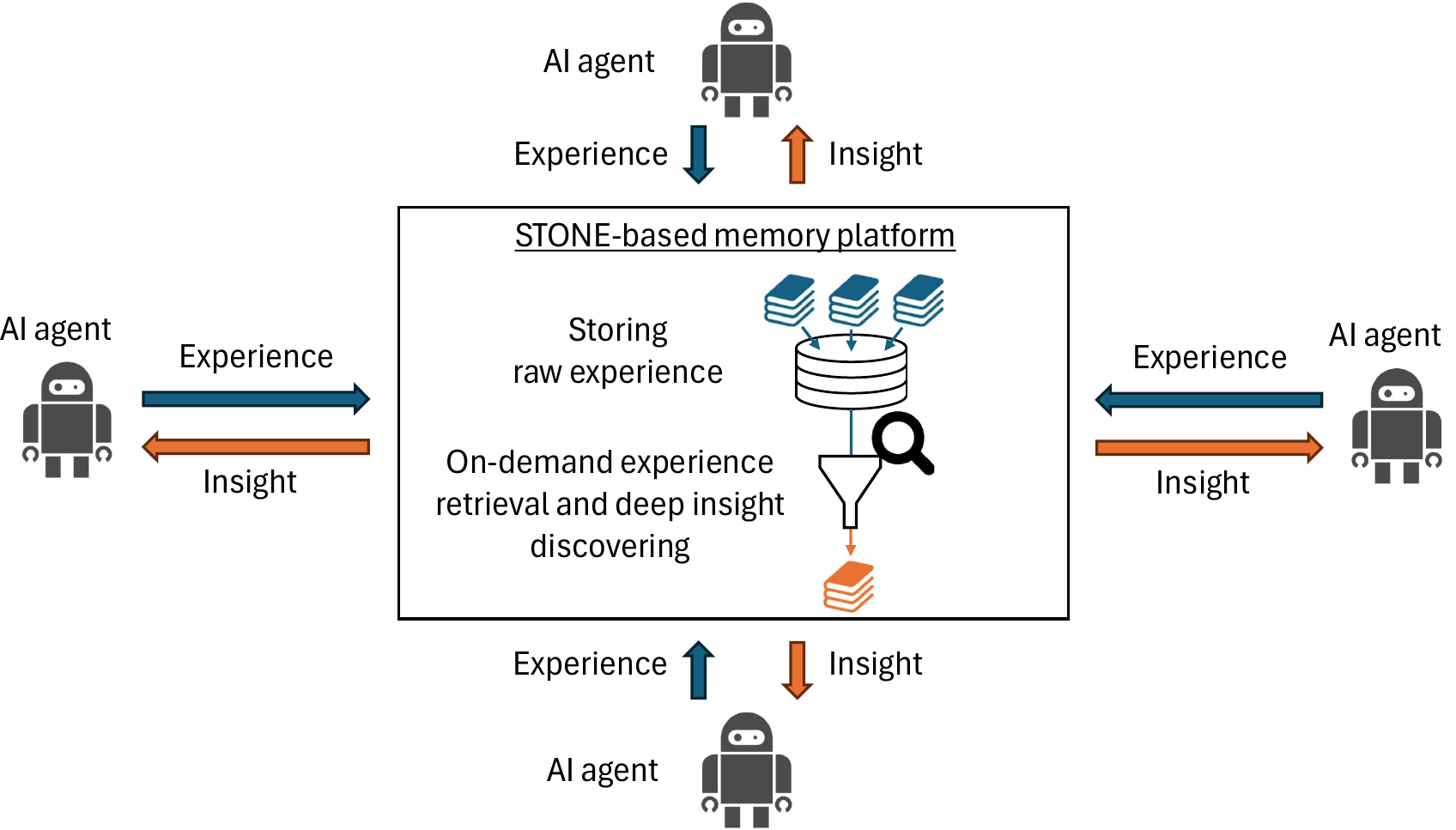}
  \end{center}
  \caption{STONE-based memory platform.} \label{fig:stoe-all}
\end{figure}

As a brief aside, even though humans possess latent learning capabilities, they may also employ an ``extract then store'' strategy, speculatively selecting which information to retain and accepting potential loss of useful details.
It may appear somewhat ironic, but imitating this approach can be advantageous for efficiency and may be appropriate when aiming for human-like AI with similarly uncertain memory.
However, technological advances such as increased computational speed and energy efficiency enable AI to overcome such limitations; thus, pursuing strategies distinct from those of humans may be essential for achieving artificial superintelligence (ASI).

Providing LLMs with information extracted from relevant experiences (e.g., task episodes) can be effective, but only when the information is consistent. 
In a stochastic environment, however, performing the same action in the same state may lead to different outcomes. 
In such cases, the information extracted from experiences lacks consistency.
Such probabilistic settings are common in real-world applications, yet most benchmarks for LLM-based agents assume deterministic environments (e.g.,~\cite{zhao_expel_2024}). 
When relevant information is inconsistent, inter-context conflict~\cite{xu_knowledge_conflict_survey_2024} may arise, leading to unpredictable LLM outputs and reducing the effectiveness of experience utilization.
For instance, if the LLM receives conflicting information A and B as context, it may base its output on A, B, or neither. 
We argue that it is not appropriate to reflect a single piece of information extracted from experiences in an agent's behavior with an effective learning rate close to 1 (as in in-context learning, ICL~\cite{brown_language_2020}).
Instead, the information provided to the LLM should be extracted through a statistical process from multiple relevant experiences.
We refer to this process as ``deeper insight discovery.''

The more experiences are collected, the more useful insights are obtained. Thus, a large amount of experiences needs to be collected. However, experience collection significantly depends on individual agents' trial and error while collected experiences are siloed within an individual agent. The agents are put under burden of trial and error, i.e., many task executions are required to obtain effective insights. We advocate that sharing experiences of multiple agents reduces per agent cost of trial and error.

In this paper, we highlight (i) \textbf{Store Then ON-demand Extract: STONE}, (ii) \textbf{deeper insight discovery}, and (iii) \textbf{experience memory sharing} whose potential benefits are widely imaginable, yet have remained largely unexplored.
These design concepts are summarized in Figure~\ref{fig:stoe-all}.
The remainder of the paper is organized as follows. Section~\ref{sec:related_works} reviews related work on these three approaches. Section~\ref{sec:formalize} formalizes the highlighted approaches. Section~\ref{sec:experiments} presents simple experiments demonstrating their effectiveness. Section~\ref{sec:challenges} discusses the challenges hindering wider adoption and suggests research topics required to address them. Finally, Section~\ref{sec:conclusion} concludes the paper.

\section{Related works} \label{sec:related_works}
\subsection{Prerequisites}
In advance to describing closely related work in the following Sections~\ref{sec:sub-multifaceted}--\ref{sec:sub-expcollect}, this subsection clarifies our focused dimensions of AI memory technology: non-parametric memory and inter-inference memory use.

\subsubsection{Non-parametric (vs. parametric) memory}
One of major memory types of AI agents is parametric memory in LLMs, whereas our focus is non-parametric memory. 
Although parametric memory can be updated by fine-tuning (e.g., LoRA~\cite{hu_lora_2022}) and provide substantial inference performance gains, parametric memory is unsuitable for learning from experience. 
If we aim to enable agents to learn continuously and on the fly, frequent and prompt parameter updates would be required, which would demand impractical computational cost and speed.
In addition, on-the-fly learning from experience is also impractical because parameter updating requires a significant amount of time.
Besides, some studies~\cite{berglund_reversal_curse_2024, andrew_icl_generalization_2025, lampinen_latent_2025} report that flexible use of parametric memory is limited because information embedded in parameters depends on the task and format used during learning.
For example, reversing a relation in the embedded information (e.g., answering the question ``Who was Aristotle's teacher?'' when the embedded information is ``Plato taught Aristotle.'') is difficult.

We focus on non-parametric memory, in which obtained information is stored in external storage. When the memory is used, necessary information in the external storage is injected into prompts as \textit{context}. 
RAG (retrieval-augmented generation)~\cite{lewis_retrieval-augmented_2021}, one of major applications of non-parametric memory, injects information related to the task from external knowledge base, to improve accuracy of LLM outputs. 
ICL~\cite{brown_language_2020} represents another common use case: instructional prompts, typically consisting of few-shot examples or demonstrations such as input-output pairs or reasoning traces, are inserted into the context to enable the LLM to adapt its parametric knowledge to the format of the target task.

\subsubsection{Inter-inference (vs. Intra-inference) memory use}
Another featured dimension of our approach is inter-inference memory use rather than intra-inference memory use. 
Intra-inference memory use mainly aims to understand long context within a single inference. 
As human brain organizes vast amount of episodic experiences in memory and efficiently retrieves relevant information across time, EM-LLM~\cite{fountas_human-inspired_2025} organizes input tokens into episodic events identified Bayesian surprise and retrieves relevant information from the event memory in the following inference process. 
Titans~\cite{behrouz_titans_2024} proposes neural long-term memory module that is updated during an inference according to surprise metric of input information defined by a loss function. 
In addition, the architecture supports forgetting mechanism. The architecture efficiently handles super long-context.

Our approach falls into inter-inference memory use because aiming to use the same experience memory across various tasks and environments, i.e., multiple inferences.
The main objective of inter-inference memory use is to learn from past experience to do tasks well in the future.
Most state-of-the-art techniques of this approach (e.g., ExpeL~\cite{zhao_expel_2024}, MIRIX~\cite{wang_mirix_2025}, RoboMemory~\cite{lei_robomemory_2025}, and ACE~\cite{zhang_agentic_2025}) distill knowledge and support retrieval mechanisms of experience memory, to encourage taking the same action in successful experience or avoiding the same action in failed experience. 
In the techniques, relevant information (i.e., non-parametric memory) is retrieved from experience memory and injected into prompts during inferences. 

\subsection{Cross-task and multifaceted use of experiences} \label{sec:sub-multifaceted}

The objective of STONE is to preserve experiences as complete memories, thereby minimizing information loss and enabling flexible and multifaceted utilization of knowledge across tasks.
Some prior works explicitly address such cross-task memory usage.

For example, ExpeL~\cite{zhao_expel_2024} mentions the ``extraction/abstraction of cross-task knowledge from these experiences'' and stores both success trajectories (raw experiences) and insights as memory.
However, in their design, success trajectories are used primarily as in-context few-shot examples to help LLMs interpret tasks, while only insights extracted based on the current task are leveraged for transferrable knowledge.
Consequently, information from experiences that is not explicitly related to the original task cannot be utilized in other tasks, even if it is potentially useful.

Similarly, FLEX~\cite{cai_flex_2025} emphasizes retaining cross-task experiences for future use in an experience library by aggregating acquired experiences. However, the library-updating algorithm is highly dependent on the task during which the experiences were obtained. Since the experiences themselves are not preserved and only the experience library is stored, FLEX cannot exhaustively utilize all the information embedded in the original experiences.

Although these methods aim to enable the use of experiences across tasks, they deliberately adopt an ``extract then store'' approach---which risks losing information that could be useful for other tasks---rather than the STONE paradigm.
This choice likely arises not from a lack of awareness about the STONE-like approach. Rather, it seems to stem from a deliberate avoidance, which implicitly suggests the inherent challenges associated with implementing the STONE paradigm.

On the other hand, GAM~\cite{yan_general_2025} is conceptually similar to STONE.
GAM regards memorization as a form of data compression, necessarily resulting in information loss.
Its method involves dividing LLM input into pages without any compression, storing them, and retrieving relevant pages during inference.
Information necessary for the current task is extracted via Just-in-Time Compilation, and answers are generated based on this extracted information.
However, their empirical evaluation relies on benchmarks such as LoCoMo~\cite{maharana_locomo_2024}, which are designed for intra-inference memory use, and their prompt design appears aligned with this setting.
The paper also omits any discussion of the scalability needed for inter-inference memory use, highlighting the broader lack of benchmarks suited to this scenario.

\subsection{Insight extraction} \label{sec:sub-insightdepth}
ExpeL~\cite{zhao_expel_2024}, AWM~\cite{wang_awm_2025}, and CER~\cite{liu_contextual_2025} enable learning from experience by extracting common patterns such as insights, skills, and dynamics from raw task trajectories and providing them to the LLM as context, while model parameters remain fixed. 
Furthermore, they include mechanisms to refine the database of common patterns through repeated trial and error.
However, in these works, the common patterns extracted from experience remain at a surface level, assuming only deterministic environments, and are insufficient for capturing the probability distributions of responses in stochastic environments.
By capturing the probability distributions and reflecting them in the agent's behavior, the usefulness of experience utilization is expected to increase; however, research in this area remains largely underexplored.

\subsection{Experience collection and sharing} \label{sec:sub-expcollect}
While most conventional methods for self-evolving AI agents collect experiences solely by individual agents, some methods enable sharing memory among multiple agents. 
In multi-agent collaboration scenarios, agents do subtasks of global tasks. 
For instance, RoboOS-Next~\cite{tan_roboos-next_2025} proposes shared memory for robots collaborating for global tasks. 
Spatio Temporal Embodiment Memory (STEM) including spatial scene geometry, temporal event history, and embodiment profiles of each agent is stored in the shared memory. 
A centralized brain agent coordinates the robots using information in the STEM to process global tasks.
In non-collaborating agent scenarios, agents share their experiences while doing individual tasks. 
MemOS~\cite{li_memos_2025}, which is a comprehensive agent memory management architecture, supports MemStore to exchange modularized knowledge among institutions and industrywide networks on a publish-subscribe mechanism. 
The exchanged knowledge is distilled from agents' experience for specific domains/tasks.
These memory sharing methods efficiently collect experience while assuming that memory-sharing agents do the same, similar, or cooperative tasks. 
Multifaceted use of shared memory (i.e., obtaining information from experiences of other agents performing completely different tasks) is not supported.

The memory sharing (MS) framework~\cite{gao_memory_2024} proposes to share experience on a single experience pool to support cross-task knowledge exchange. 
Agent experience, which is a question and answer pair, is scored according to pre-defined criteria and the pair with the score greater than the threshold is stored in the pool. 
Although the MS framework aims to provide globally useful experience pool, the generalization ability of shared memory is limited due to the experience scoring by the global criteria. 
Experience needs to be evaluated by task-specific criteria to be useful for the task.

\section{Formalization of the Paradigms} \label{sec:formalize}

In this section, we formally and succinctly characterize the paradigms (i) \textbf{Store Then ON-demand Extract: STONE}, (ii) \textbf{deeper insight discovery}, and (iii) \textbf{experience memory sharing}, highlighting how each differs from conventional, widely adopted paradigms.

\subsection{STONE vs. extract then store}

Both paradigms determine how experience acquired during a task is transformed into persistent memory used to support future tasks.

Consider an agent performing a sequence of tasks $\mathcal{T}_1$, $\mathcal{T}_2$, $\cdots$.
When the agent executes task $\mathcal{T}$, it obtains an experience $\mathcal{E}$ (for example, inputs, outputs, trajectories, and observations).
Let $\mathcal{S}$ denote the memory storage, which contains a growing set of stored memory entries.
Over time, $\mathcal{S}$ becomes $\mathcal{S} = \{\mathcal{M}_1, \mathcal{M}_2, \cdots\}$, where each $\mathcal{M}_i$ is a stored memory derived from past experiences through task $\mathcal{T}_i$.

\textbf{In the extract-then-store paradigm}, when experience $\mathcal{E}$ is acquired during task $\mathcal{T}$, the agent extracts only information deemed useful for $\mathcal{T}$.
Let $f_\mathcal{T}(\mathcal{E})$ be the extraction function for task $\mathcal{T}$, and this is typically executed by an LLM.
The stored memory entry is
\begin{equation}
\mathcal{M} = f_\mathcal{T}(\mathcal{E}),
\end{equation}
and the memory storage is updated as 
\begin{equation}
\mathcal{S} \leftarrow \mathcal{S} \cup {M}.
\end{equation}
During a later task $\mathcal{T}'$, the agent retrieves the entries relevant to $\mathcal{T}'$ from $\mathcal{S}$ and uses them for inference.

\textbf{In the STONE paradigm}, the agent stores the full experience without task-specific filtering.
That is,
\begin{equation}
\mathcal{M} = \mathcal{E},
\end{equation}
and
\begin{equation}\label{eq:stone}
\mathcal{S} \leftarrow \mathcal{S} \cup \mathcal{M} (= \mathcal{S} \cup {\mathcal{E}}).
\end{equation}
When a new task $\mathcal{T}'$ is encountered, relevant stored experiences are retrieved from $\mathcal{S}$, and the extraction function $f_{T'}$ is applied at retrieval time to obtain the task-useful information. Again, $f_{T'}$ is typically executed by an LLM.

The key limitation of extract-then-store paradigm arises because $f_\mathcal{T}(\mathcal{E})$ preserves only information useful for the current task $\mathcal{T}$.
If $\mathcal{E}$ contains information $x$ that is not useful for $\mathcal{T}$ but is useful for a future task $\mathcal{T}'$, then $x$ is lost whenever $x$ is not included in $f_\mathcal{T}(\mathcal{E})$.
In contrast, STONE stores the entire experience $\mathcal{E}$ regardless of its immediate utility, thereby eliminating the risk of discarding information that may become valuable for future tasks.

Besides, under the requirement that no information potentially useful for any future task may be lost, the STONE paradigm achieves the smaller possible memory size than the extract-then-store paradigm.
The proof of this statement is in appendix~\ref{sec:appendix_a}.

\subsection{Deeper insight discovery vs. simple experience replay} \label{sec:did-vs-ser}

Given a task $\mathcal{T}$, the agent retrieves information from $\mathcal{S}$ and generates an action.

\subsubsection*{Simple experience replay}

The baseline approach retrieves a single memory entry most relevant to $\mathcal{T}$:
\begin{equation}
\mathcal{M}^{*} = \arg\max_{\mathcal{M} \in \mathcal{S}} \mathrm{Rel}(\mathcal{M},\mathcal{T}),
\end{equation}
where $\mathrm{Rel}(\cdot)$ is a relevance scoring function.
The agent then extracts useful information $\mathcal{I}$ for task $\mathcal{T}$ from $\mathcal{M}^{*}$:
\begin{equation}
\mathcal{I} = f_\mathcal{T}({\mathcal{M}^{*}}).
\end{equation}
Then, the agent generates an action for task $\mathcal{T}$ using $\mathcal{I}$.
Because only one experience is used, the behavior may reflect noisy or atypical past events.

\subsubsection*{Deeper insight discovery}

The deeper insight method retrieves all memory entries with relevance above a threshold:
\begin{equation}\label{eq:comprehensive_recall}
\mathcal{M}_{\mathcal{T}} = \{\, \mathcal{M} \in \mathcal{S} \mid |f_{\mathcal{T}}(\mathcal{M})| > 0 \,\},
\end{equation}
where $|f_{\mathcal{T}}(\mathcal{M})| > 0$ means $\mathcal{M}$ has some useful information for task $\mathcal{T}$, and aggregates information across them:
\begin{equation}
\mathcal{I} = \mathrm{Discover}_{\mathcal{T}}\bigl(\{f_{\mathcal{T}}(\mathcal{M}) \mid \mathcal{M} \in \mathcal{M}_{\mathcal{T}}\}\bigr),
\end{equation}
where $\mathrm{Discover}(\cdot)$ distills useful information for $\mathcal{T}$ by aggregating information in $\mathcal{M}_{\mathcal{T}}$.
By leveraging multiple related experiences, this method captures underlying statistical structure and supports more reliable decision-making in probabilistic settings.

\subsection{Memory sharing vs. self-collection}

We consider two settings for constructing the memory storage $\mathcal{S}$ containing experiences.

\textbf{In the self-collection setting}, a single agent $A$ accumulates its own experiences.  
Let $\mathcal{M}_{A}$ denote the set of experiences collected by agent $A$.  
The memory storage is defined as
\begin{equation}
\mathcal{S} = \mathcal{M}_{A}.
\end{equation}
Thus, the diversity and volume of stored memory entries are limited by the trajectory of a single agent.

\textbf{In the memory-sharing setting}, multiple agents $\{A_{1}, \ldots, A_{K}\}$ contribute to a shared memory storage.
Let $\mathcal{M}_{i}$ denote the experiences collected by agent $A_{i}$.  
The shared memory storage is given by
\begin{equation}
\mathcal{S} = \bigcup_{i=1}^{K} \mathcal{M}_{i}.
\end{equation}
Because experience acquisition is distributed across $K$ agents, this setting provides a significantly higher rate of memory accumulation and greater experiential diversity.

\section{Simple Experiments Demonstrating the Superiority of the Highlighted Approaches} \label{sec:experiments}

\subsection{STONE}

\begin{figure}
  \centering
  \begin{minipage}[b]{0.45\columnwidth}
    \centering
    \begin{algorithm}[H]
      \caption{STONE}
      \begin{algorithmic}
        \STATE $\mathrm{budget} \leftarrow 100$, $\mathrm{CORRECT} \leftarrow 0$, $\mathcal{M} \leftarrow \emptyset $
        \FOR{$q \in \mathrm{Questions}$}
        \WHILE{$\mathrm{TRUE}$}
        \STATE $\mathcal{E}_{q} \leftarrow s_\mathcal{S}{(q)}$ \textit{\# search}
        \STATE $\mathcal{M}_{q} \leftarrow f_{q}{(\mathcal{E}_{q})}$ \emph{\textit{\# extract}}
        \IF{$\mathrm{budget} < 0$ or sufficient information $\in \mathcal{M}_{q}$}
        \STATE $\mathrm{answer} \leftarrow g(q | \mathcal{M}_{q})$ \textit{\# generate answer}
        \STATE break
        \ELSE
        \STATE $\mathrm{budget} \leftarrow \mathrm{budget} - 1$
        \STATE $\mathcal{E}_{q} \leftarrow \mathrm{retrieve}_{\mathrm{ext}}(q)$
        \STATE 
        \STATE $\mathcal{S} \leftarrow \mathcal{S} \cup \mathcal{E}_{q}$ \emph{\textit{\# store}}
        \ENDIF
        \ENDWHILE
        \IF {answer is correct}
        \STATE $\mathrm{CORRECT} \leftarrow \mathrm{CORRECT} + 1$
        \ENDIF
        \ENDFOR
      \end{algorithmic}
    \end{algorithm}
  \end{minipage}
  \hspace{0.04\columnwidth}
  \begin{minipage}[b]{0.45\columnwidth}
    \centering
    \begin{algorithm}[H]
      \caption{Extract then store}
      \begin{algorithmic}
        \STATE $\mathrm{budget} \leftarrow 100$, $\mathrm{CORRECT} \leftarrow 0$, $\mathcal{M} \leftarrow \emptyset $
        \FOR{$q \in \mathrm{Questions}$}
        \WHILE{$\mathrm{TRUE}$}
        \STATE $\mathcal{M}_{q} \leftarrow s_\mathcal{S}{(q)}$ \textit{\# search}
        \STATE 
        \IF{$\mathrm{budget} < 0$ or sufficient information $\in \mathcal{M}_{q}$}
        \STATE $\mathrm{answer} \leftarrow g(q | \mathcal{M}_{q})$ \textit{\# generate answer}
        \STATE break
        \ELSE
        \STATE $\mathrm{budget} \leftarrow \mathrm{budget} - 1$
        \STATE $\mathcal{E}_{q} \leftarrow \mathrm{retrieve}_{\mathrm{ext}}(q)$
        \STATE $\mathcal{M}_{q} \leftarrow f_{q}{(\mathcal{E}_{q})}$ \emph{\textit{\# extract}}
        \STATE $\mathcal{S} \leftarrow \mathcal{S} \cup \mathcal{M}_{q}$ \emph{\textit{\# store}}
        \ENDIF
        \ENDWHILE
        \IF {answer is correct}
        \STATE $\mathrm{CORRECT} \leftarrow \mathrm{CORRECT} + 1$
        \ENDIF
        \ENDFOR
      \end{algorithmic}
    \end{algorithm}    
  \end{minipage}
  \caption{Pseudo-algorithms for the simplified experiment. $s_\mathcal{S}(q)$ retrieves the memory entry relevant to $q$ from storage $\mathcal{S}$. $f_{q}{(\cdot)}$ extracts information relevant to $q$. $g(q | \cdot)$ generates an answer to $q$ by referring to given information. $\mathrm{retrieve}_{\mathrm{ext}}(q)$ retrieves documents relevant to $q$ from the external database in the simplified experiment setting. More generally, this process corresponds to acquiring new experience.} \label{fig:exp_stone_alg}
\end{figure}

\begin{figure}
  \centering
  \begin{minipage}[b]{0.6\columnwidth}
    \centering
      \includegraphics[width=\columnwidth]{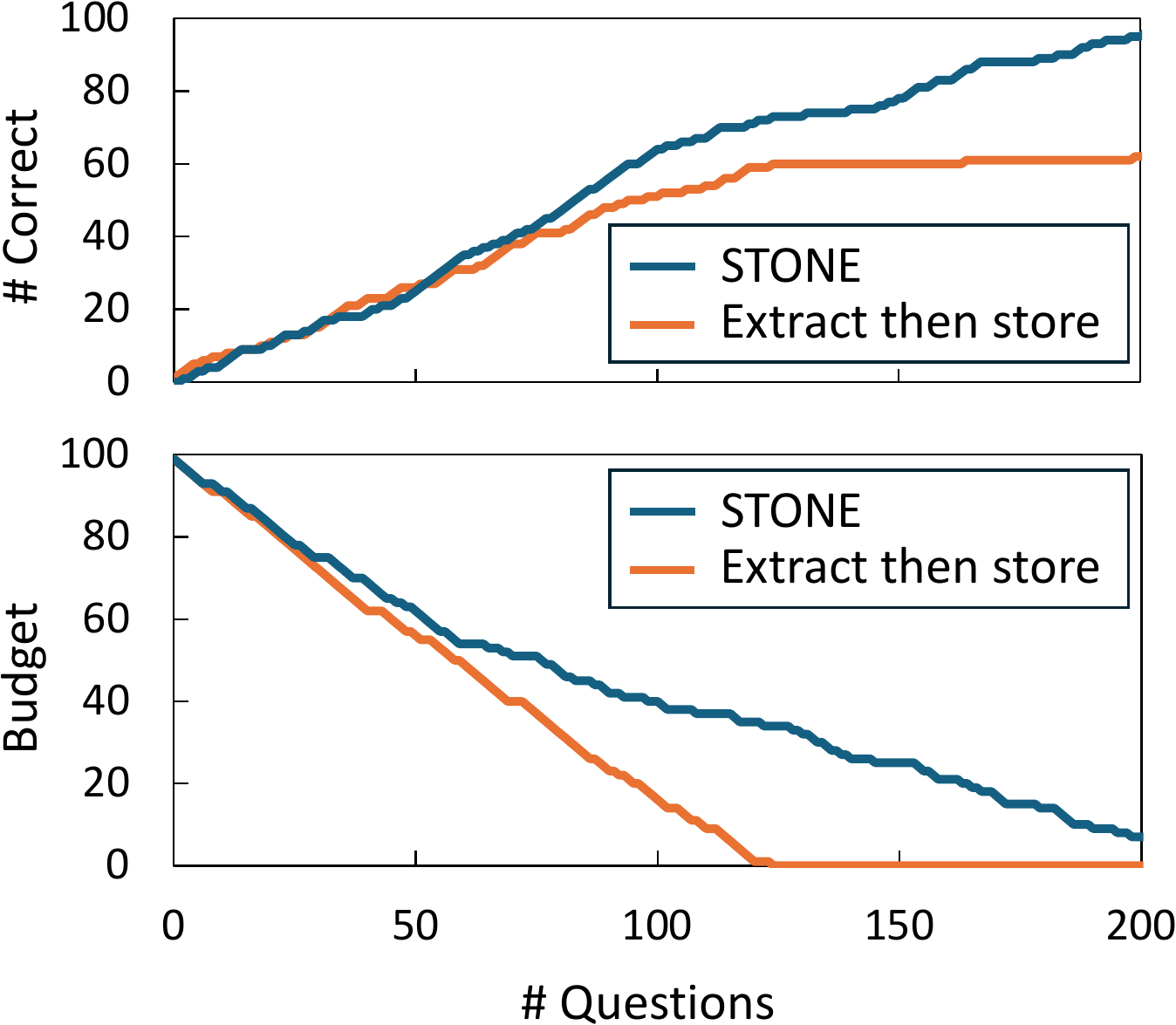}
  \end{minipage}
  \caption{STONE (store then on-demand extract) vs. extract then store.} \label{fig:exp_stone}
\end{figure}

Conventional approaches employing the ``Extract Then Store'' paradigm only retain information deemed relevant to the task at the time of experience acquisition. Consequently, information potentially useful for future tasks is often discarded, limiting the system's capability for flexible knowledge reuse.
In contrast, we highlight the advantage of the ``Store Then On-demand Extract'' (STONE) paradigm, in which the agent preserves the raw experience (e.g., an entire document) and extracts task-relevant information as needed during task execution. This approach prevents information loss and facilitates more versatile utilization of previously acquired knowledge.

To empirically demonstrate the benefits of STONE, we conducted a question answering task concerning internal company regulations.
Multiple documents detailing company policies are provided.
For each document, we generated several questions and corresponding correct answers using a \texttt{gpt-oss-120b}, ensuring that each question could be accurately answered by referencing its respective document.
The agent is requested to solve these questions in random order.
In this scenario, accessing the external database of internal regulations to retrieve a document is considered an ``experience.'' The agent's retrieval budget is limited to 100 experiences.

\textbf{Extract Then Store}: For each question, the agent first searches its memory. If sufficient information is unavailable, it retrieves the relevant document from the external database, at the cost of one experience. However, it only stores information that is directly necessary to answer the current question; all other information is discarded. As a result, if subsequent questions require other information from the same document, the agent must retrieve it again, rapidly depleting its retrieval budget.

\textbf{STONE}: Whenever the agent accesses an external document, it stores the entire document as-is. This allows the agent to answer later questions using previously retrieved documents, thereby reducing redundant retrievals and conserving its budget.

The detailed experiment setting is illustrated in Figure~\ref{fig:exp_stone_alg}.
The top plot of Figure~\ref{fig:exp_stone} shows the agent's cumulative correct answers against the number of questions, while the bottom of Figure~\ref{fig:exp_stone} shows the rate at which the retrieval budget is consumed. ``Extract Then Store'' methods exhaust the retrieval budget more quickly. Once depleted, the agent must rely exclusively on its limited memory, which significantly impairs performance on subsequent tasks. In contrast, STONE enables more efficient budget usage and prolonged learning, allowing the agent to continue answering questions either using stored documents or by retrieving new ones as necessary.
These results demonstrate that STONE enables multifaceted and flexible use of stored experiences.

\subsection{Deeper insight discovery}
\begin{figure}
  \centering
  \begin{minipage}[b]{0.49\columnwidth}
    \centering
    \includegraphics[width=0.9\columnwidth]{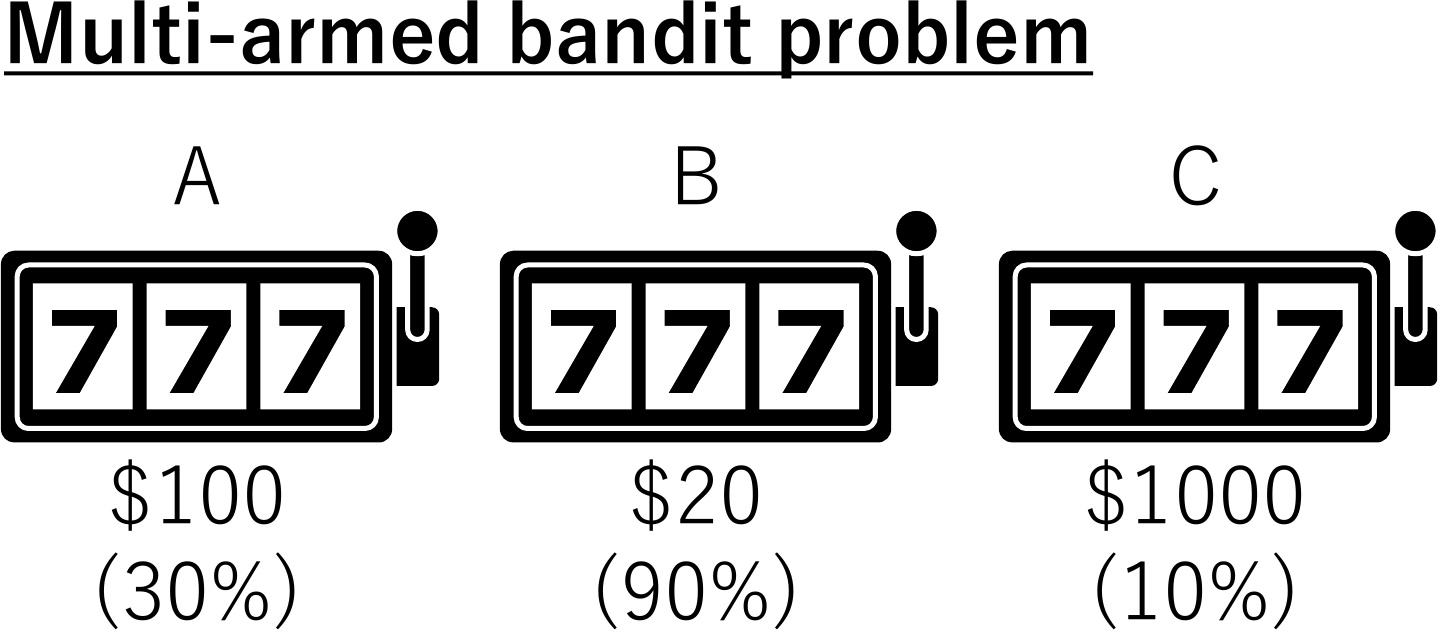}
    \subcaption{Settings.}\label{fig:exp_deep_settings}
  \end{minipage}
  \begin{minipage}[b]{0.49\columnwidth}
    \centering
    \includegraphics[width=0.9\columnwidth]{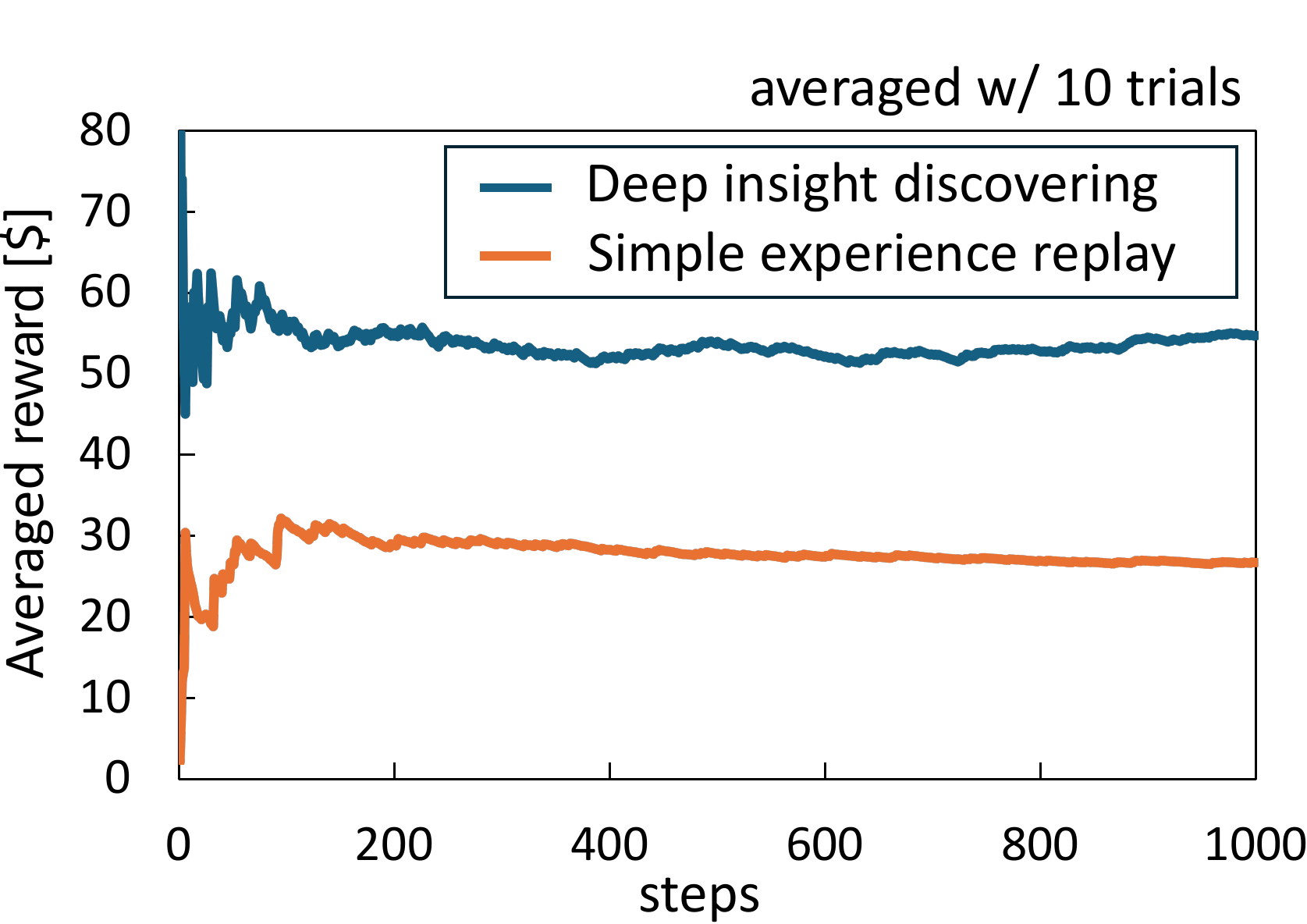}
    \subcaption{Rewards for each step.}\label{fig:exp_deep_res}
  \end{minipage}
  \caption{Deeper insight discovery vs. simple experience replay.} \label{fig:exp_deep}
\end{figure}

To demonstrate the superiority of ``deeper insight discovery'' over simple experience replay, we designed a straightforward experiment using the multi-armed bandit problem---a commonly used benchmark in reinforcement learning. In this task, as illustrated in Figure~\ref{fig:exp_deep_settings}, three arms are assigned different probabilities of success and rewards, with the objective to maximize cumulative reward.

When applying simple experience replay to this setting, an agent may select the same arm again after a success, and switch arms following a failure. However, as noted in standard reinforcement learning literature, many approaches outperform such naive strategies.

Here, we adopted the $\varepsilon$-greedy policy, which can be easily applied in the context of experience utilization. The agent estimates the expected value of each arm based on observed outcomes, selecting the arm with the highest estimated value with probability $1-\varepsilon$, and otherwise selecting randomly among the remaining arms.

Figure~\ref{fig:exp_deep} presents the average reward over ten trials, clearly showing that the $\varepsilon$-greedy policy achieves substantially higher rewards compared to simple experience replay. While this result may appear trivial, it strongly suggests that statistical processing and leveraging a broader range of experiences---as embodied by \textbf{deeper insight discovery}---offers significant potential for performance improvement over conventional approaches focused solely on repeating successful experiences or avoiding failed ones.

\subsection{Experience memory sharing}
\begin{figure}
  \centering
  \begin{minipage}[b]{0.49\columnwidth}
    \centering
    \includegraphics[width=0.9\columnwidth]{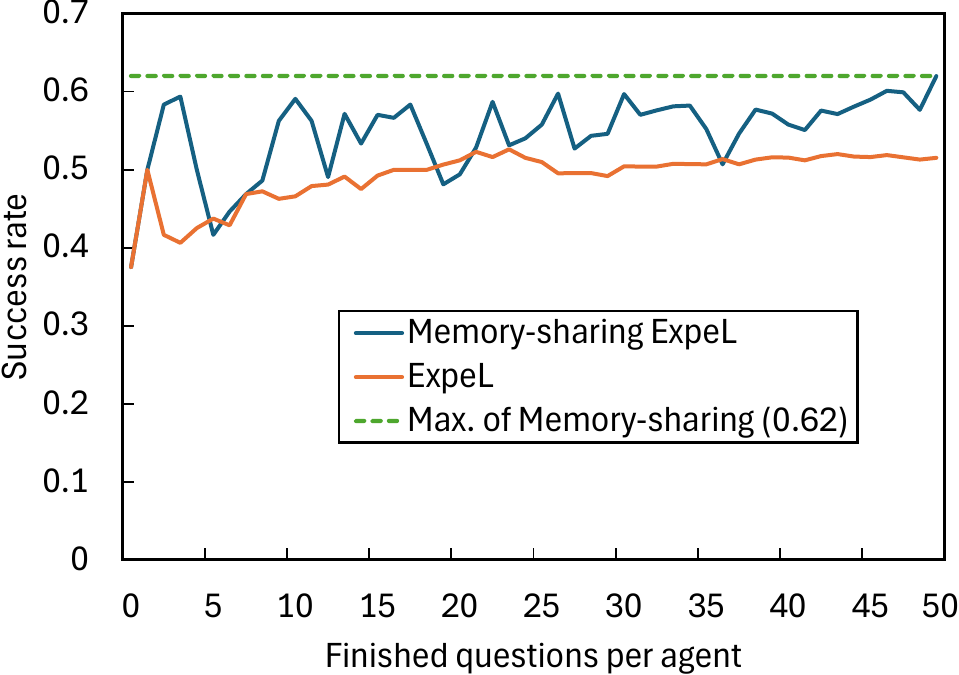}
    \subcaption{Success rate on 1--50 questions.}\label{fig:memsh-expel-50}
  \end{minipage}
  \begin{minipage}[b]{0.49\columnwidth}
    \centering
    \includegraphics[width=0.9\columnwidth]{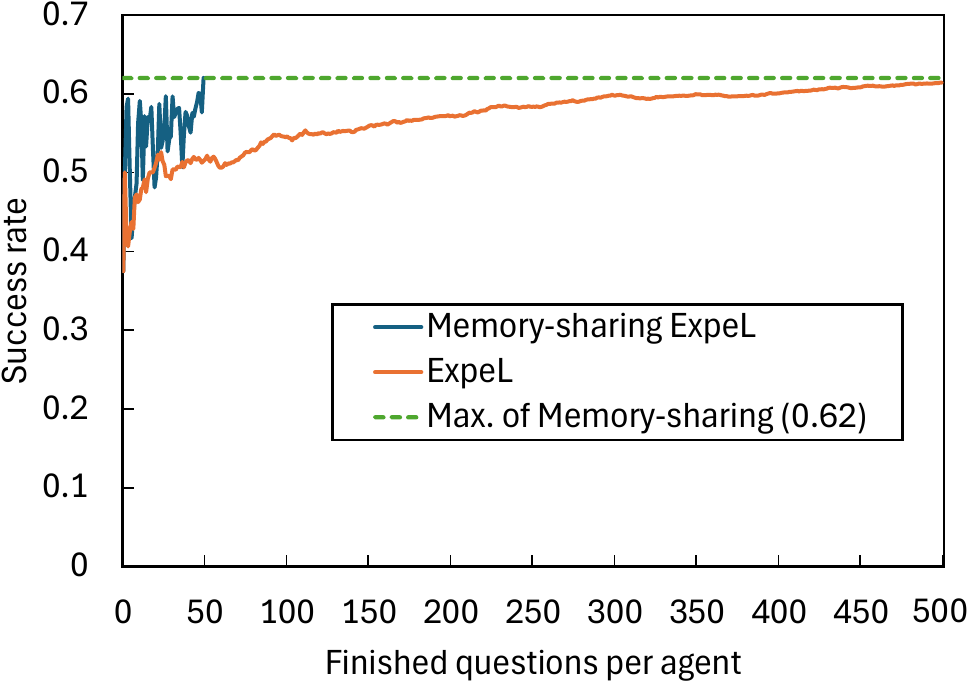}
    \subcaption{Success rate on 1--500 questions.}\label{fig:memsh-expel-500}
  \end{minipage}
  \caption{With memory-sharing vs. without memory-sharing.} \label{fig:memsh-expel}
\end{figure}

In the memory-sharing platform, a large volume of experience data can be collected efficiently because multiple AI agents can simultaneously store their experiences.
We demonstrate the efficiency of memory sharing by expanding the ExpeL~\cite{zhao_expel_2024} agent's functionality to share its trajectories and extracted rules.
We conducted an experiment in which an agent processed HotpotQA questions while its success rate was measured after each question. 
The success rate is defined as the number of succeeded questions divided by the total number of previously answered questions. 
Figure~\ref{fig:memsh-expel-50} and Figure~\ref{fig:memsh-expel-500} show the number of finished questions plotted against the success rates on 1--50 questions (an enlargement of Figure~\ref{fig:memsh-expel-500}) and 1--500 questions, respectively. For ``Memory-sharing ExpeL'', ten agents processed 50 questions in the individually shuffled orders while sharing success trajectories of similar questions and rules extracted from trajectories of the agents. 
At each round, the ten agents processed one question in a shuffled order, and this was repeated for 50 rounds. 
The success rate of the last agent at each round is plotted, with the last agent having access to the memory of all previous experiences as well as all other agents in the current round. 
This is a fair setting to ``ExpeL'' experiments because an ExpeL agent can always access all the memory of previous experiences.
For ``ExpeL'', one agent processed 500 questions in the interleaved order of the ten shuffles used in the ``Memory-sharing ExpeL'' experiment. 
For ``Memory-sharing ExpeL'' and ``ExpeL'', the average values over eight times of the same experiments are shown in Figure~\ref{fig:memsh-expel}. 
The success rate of the memory-sharing agents reached 0.62 when they processed 50 questions per agent (Figure~\ref{fig:memsh-expel-50}). 
On the other hand, the ExpeL agent required processing all the 500 questions to reach the success rate 0.62 (Figure~\ref{fig:memsh-expel-500}). 
The ratio of required 500 questions to 50 questions was the same as that of ten memory-sharing agents to the one ExpeL agent. 
This result indicates that memory-sharing capability reduces the burden of trial-and-error per agent while trajectories are efficiently collected.

\section{Challenges} \label{sec:challenges}

In this section, we discuss several challenges that may have hindered the adoption of \textbf{Store Then ON-demand Extract: STONE}, \textbf{deeper insight discovery}, and \textbf{experience memory sharing} despite their potential usefulness, and that may help explain why they have remained underexplored.

\subsection{Storage capacity}
In STONE, experiences are stored in raw form, as shown in Eq.~(\ref{eq:stone}).
Under the requirement that no information potentially useful for any future task may be lost, this strategy minimizes memory usage (appendix~\ref{sec:appendix_a}); nevertheless, compared with methods that store only information fragments relevant to the current task, it results in substantially larger storage demands.
A promising direction is the lossless compression of experiences and the elimination of redundancies with previously stored memory.
However, what we emphasize most is increasing the density and capacity of the storage devices themselves.
In this direction, for example, 3D-flash memory with a bit density of $37.6~\mathrm{Gb}/\mathrm{mm}^{2}$~\cite{thimmaiah_bics_2026}, and SSDs with capacities of 245.76 TB in both a standard 2.5-inch and EFSFF E3.L form factors~\cite{kioxia_ssd_2025} have been demonstrated.
Note that memory sharing schemes also help reduce the overall storage demand by avoiding redundant storage for experiences that can be leveraged by multiple AIs.

\subsection{Inference performance}
In STONE, a major concern is a significant runtime overhead, as useful information must be extracted from raw experiences during task execution.
This extraction is typically performed by prompting an LLM with the raw experience, the current task description, and an instruction to extract task-relevant information.
Because the computational cost of LLM inference generally scales quadratically with the context length, longer inputs lead to substantially higher latency.
Storing raw experiences as memory entries therefore results in very large contexts, further exacerbating this cost.

A possible approach to accelerate insight extraction is utilizing KV (key value)-cache enabling LLMs to avoid long experience data processing at runtime. 
KV-cache is a store of key-value expressions of tokens processed by an LLM~\cite{luohe_keep_2024}. 
Reuse of KV-cache of long experience data allows the LLM to start from states of being aware of the long experience without reprocessing the experience.
Different from summarization and abstraction that are suitable to human memory, KV-cache retains all information in processed tokens, and we argue that KV-cache is the optimal memory form for AI. 
While current primary target of KV-cache utilization is long session and multi-turn conversations for chat service~\cite{nvidia_motivation_2025}, we need to research and develop algorithms and software on this technology optimized for STONE paradigm.
For both of them, the development of ultra-high IOPS SSD~\cite{kioxia_ultra_high_iops_2025} is expected to significantly broaden its scope of application and enhance its effectiveness.

\subsection{Comprehensive recall}
For deep insight discovery, it is required to reliably recall all relevant experiences that have useful information for the current task (Eq. \ref{eq:comprehensive_recall}). 
Similarity search using dense vectors (approximate nearest neighbor search: ANNS) is popularly used by AI agents (e.g., agentic RAG~\cite{singh_agentic_2025}) to find useful experiences in the database. 
However, ANNS is not appropriate for this comprehensive recall~\cite{weller_theoretical_2025}, because ANNS is designed to retrieve the most similar items.

Semantic logical search techniques using sparse vector format can be a fundamental technology of comprehensive recall because of its higher interpretability of data~\cite{borges_kale_2023,park_decoding_2025}.
Furthermore, software or devices suitable for large-scale logical search using sparse vector representations need to be developed. While research and development on software and hardware for dense vector search has been highly active (e.g.,~\cite{tatsuno_aisaq_2025, yiwei_ansmet_2025}), we hope to see similar progress for logical search based on sparse vector formats.

\subsection{Machine learning to discover deeper insight}
As described in Section~\ref{sec:did-vs-ser}, a discovery function ($\mathrm{Discover}(\cdot)$) is required to capture the statistical structure in large amounts of experience in stochastic environments and obtain useful deeper insights.
This is most reasonably achieved through machine learning.
Specifically, a Q-function that derives the return $G$ from the current state $s$ (and the past history $h$, if necessary) and the subsequent action $a$ becomes the deeper insight.
However, for LLM-based agents, the open-ended language action space makes the probability of revisiting exactly the same action extremely low; as a result, the reward variance for each action becomes large, leading to unstable learning~\cite{yang2025ariatraininglanguageagents}.
Additionally, as a pre-processing step, experiences must be vectorized using a pretrained embedding model, but there is also concern that the performance of the embedding model may become a bottleneck for the overall process.
Furthermore, even if the Q-function can be estimated accurately, a mechanism is required to ensure that the LLM agent reliably takes actions that follow the Q-function.
These challenges must be addressed in order to enable machine learning-based deeper insight discovery.

\subsection{Infrastructures for memory sharing platform}
For effective memory sharing and transfer, Cheng et al.~\cite{yihua_cheng_large_2024} proposed the concept of Knowledge Delivery Networks (KDN), which deploy distributed KV-cache storage across the Internet, drawing an analogy to Content Delivery Networks (CDN).  
Such an approach highlights the need for scalable and globally accessible infrastructures to support large-scale memory platforms, enabling efficient retrieval, low-latency access, and seamless sharing of learned knowledge across multiple agents and systems.
In addition to globally accessible infrastructures such as KDN, local infrastructures (i.e., edge computing) supporting memory platforms are also necessary for experience sharing in limited space.

\subsection{Security and privacy}
Memory sharing inherently raises concerns regarding privacy and security.
Experience may contain personal or sensitive information whose disclosure should be strictly limited.
The shared memory pool may be contaminated by the introduction of misinformation or harmful data, including by malicious users.
Compliance with relevant regulations, such as the General Data Protection Regulation (GDPR), must be carefully considered for practical deployment.
Research on memory sharing that addresses security and privacy concerns is beginning to emerge~\cite{alireza_security_2025}.
But as with this paper, most studies on memory sharing tend to treat security and privacy as future work or limitations.
However, moving forward, it will be essential to develop memory-sharing approaches that thoroughly address security and privacy concerns.

\section{Conclusion} \label{sec:conclusion}
In this paper, we have identified three under-explored yet highly promising directions for advancing self-evolving AI agents equipped with external memory. Furthermore, we showed the effectiveness of the directions using simple experiments.
(i)~Store-Then-ON-demand-Extract (STONE): By persisting raw experiences instead of pre-selected extracts, STONE preserves latent information that can be leveraged flexibly across heterogeneous future tasks, thus overcoming the inherent lossiness of ``extract-then-store'' pipelines.
(ii)~Deeper insight discovery: We introduced a process to extract or learn useful information from many experiences, rather than from a single experience. This mitigates inter-context conflicts in non-deterministic settings and guides the agent toward more reliable decision-making.
(iii)~Experience Memory Sharing: We demonstrated that pooled experience repositories enable agents to amortize collection costs.
We also identify key technical challenges, including storage capacity, comprehensive recall, security, etc.
We anticipate that continued research along these lines will pave the way for AI agents that continuously and adaptively improve themselves, broadening the scope of intelligent behavior in complex, dynamic environments.

\bibliographystyle{unsrt}  
\bibliography{bib-stoe-arxiv}  

\appendix 
\section{Memory size of STONE is minimal.}\label{sec:appendix_a}

\theoremstyle{plain}
\newtheorem{thm}{Theorem}
\begin{thm}
Under the requirement that information potentially useful for any future task may not be lost and the assumption that the memory size of extracted information is the same to that of original experience, the STONE paradigm achieves the smallest possible memory size among all extract-then-store paradigms.
\end{thm}

\begin{proof}
Let $\mathcal{E}$ be the set of acquired experience, and let
$f_{\mathcal{T}_{i}}(\mathcal{E}) \subseteq \mathcal{E}(1\le i\le k)$ denote the extracted useful information for the future tasks $\mathcal{T}_{1},\cdots,\mathcal{T}_k$.
Their memory sizes satisfy $|f_{\mathcal{T}_{i}}(\mathcal{E})|\le |\mathcal{E}|$ often with strict inequality.

In STONE paradigm, memory $\mathcal{M}_{\mathrm{STONE}}$ is all experience:
\[
\mathcal{M}_{\mathrm{STONE}} = \mathcal{E},
\qquad
|\mathcal{M}_{\mathrm{STONE}}| = |\mathcal{E}|.
\]

In extract-then-store (ETS) paradigms, memory $\mathcal{M}_{\mathrm{ETS}}$ is the tuple of extracted information for the tasks:
\[
\mathcal{M}_{\mathrm{ETS}}
= (f_{\mathcal{T}_{1}}(\mathcal{E}),\cdots,f_{\mathcal{T}_{k}}(\mathcal{E})).
\]

The zero future-loss requirement demands that every
$x\in\mathcal{E}$ that could be useful for any future task must be stored because no $x \in \mathcal{E}$ can be guaranteed irrelevant to all future tasks. $\mathcal{M}_{\mathrm{ETS}}=(f_{\mathcal{T}_{1}}(\mathcal{E}),\cdots,f_{\mathcal{T}_{k}}(\mathcal{E}))$ includes information corresponding to any $x\in\mathcal{E}$ while information in $f_{\mathcal{T}_{i}}(\mathcal{E})$ and $f_{\mathcal{T}_{j}}(\mathcal{E})$ for different tasks $\mathcal{T}_i$ and $\mathcal{T}_j(i\neq j,1\le i,j\le k)$ can be overlapped and redundant. Thus, the memory size $|\mathcal{M}_{\mathrm{ETS}}|$ satisfies 
\[
|\mathcal{M}_{\mathrm{ETS}}| \ge |\mathcal{E}|.
\]

Hence, STONE achieves the minimum possible memory size:
\[
|\mathcal{M}_{\mathrm{STONE}}| \le |\mathcal{M}_{\mathrm{ETS}}|.
\]

\end{proof}

\end{document}